\documentclass[letterpaper,11pt]{article}

\usepackage{diagbox}
\usepackage{makecell}
\usepackage{booktabs}
\usepackage{tikz}
\usepackage{amsmath}
\usepackage{amsfonts}
\usepackage{amssymb}
\usepackage{amsthm}
\usepackage{url,ifthen}
\usepackage{enumitem}
\usepackage{srcltx}
\usepackage{multirow}
\usepackage{boxedminipage}
\usepackage[margin=1in]{geometry}
\usepackage{nicefrac}
\usepackage{xspace}
\usepackage{graphicx}
\usepackage{color}
\usepackage{subfigure}
\usepackage{colortbl}
\usepackage{setspace}
\usepackage{natbib}
\usepackage{dsfont}
\usepackage{algorithmic}
\usepackage{algorithm}
\usepackage{authblk}
\setcitestyle{square, authoryear}
\definecolor{DarkGreen}{rgb}{0.1,0.5,0.1}
\definecolor{DarkRed}{rgb}{0.5,0.1,0.1}
\definecolor{DarkBlue}{rgb}{0.1,0.1,0.5}
\definecolor{Gray}{rgb}{0.2,0.2,0.2}

\usepackage{listings}
\lstdefinestyle{mystyle}{
    commentstyle=\color{DarkBlue},
    keywordstyle=\color{DarkRed},
    numberstyle=\tiny\color{Gray},
    stringstyle=\color{DarkGreen},
    basicstyle=\footnotesize,
    breakatwhitespace=false,         
    breaklines=true,                 
    captionpos=b,                    
    keepspaces=true,                 
    numbers=left,                    
    numbersep=5pt,                  
    showspaces=false,                
    showstringspaces=false,
    showtabs=false,                  
    tabsize=2
}
\usepackage[small]{caption}
\usepackage[pdftex]{hyperref}
\hypersetup{
    unicode=false,          
    pdftoolbar=true,        
    pdfmenubar=true,        
    pdffitwindow=false,      
    pdfnewwindow=true,      
    colorlinks=true,       
    linkcolor=DarkBlue,          
    citecolor=DarkGreen,        
    filecolor=DarkRed,      
    urlcolor=DarkBlue,          
    pdftitle={},
    pdfauthor={},
}

\newcommand{\algorithmlabel}[1]{\label{alg:#1}}
\newcommand{\algorithmref}[1]{\hyperref[alg:#1]{Algorithm~\ref{alg:#1}}}

\newcommand{\chapterref}[1]{\hyperref[ch:#1]{Chapter~\ref{ch:#1}}}

\newcommand{\claimref}[1]{\hyperref[claim:#1]{Claim~\ref{claim:#1}}}

\newcommand{\corollaryref}[1]{\hyperref[cor:#1]{Corollary~\ref{cor:#1}}}
\newcommand{\definitionlabel}[1]{\label{def:#1}}
\newcommand{\definitionref}[1]{\hyperref[def:#1]{Definition~\ref{def:#1}}}
\newcommand{\equationlabel}[1]{\label{eq:#1}}
\newcommand{\equationref}[1]{\hyperref[eq:#1]{Equation~\ref{eq:#1}}}

\newcommand{\factref}[1]{\hyperref[fact:#1]{Fact~\ref{fact:#1}}}
\newcommand{\figurelabel}[1]{\label{fig:#1}}
\newcommand{\figureref}[1]{\hyperref[fig:#1]{Figure~\ref{fig:#1}}}
\newcommand{\tablelabel}[1]{\label{tab:#1}}
\newcommand{\tableref}[1]{\hyperref[tab:#1]{Table~\ref{tab:#1}}}

\newcommand{\itemref}[1]{\hyperref[item:#1]{Item~(\ref{item:#1})}}
\newcommand{\lemmalabel}[1]{\label{lem:#1}}
\newcommand{\lemmaref}[1]{\hyperref[lem:#1]{Lemma~\ref{lem:#1}}}

\newcommand{\propref}[1]{\hyperref[prop:#1]{Proposition~\ref{prop:#1}}}
\newcommand{\propositionlabel}[1]{\label{prop:#1}}
\newcommand{\propositionref}[1]{\hyperref[prop:#1]{Proposition~\ref{prop:#1}}}

\newcommand{\remarkref}[1]{\hyperref[rem:#1]{Remark~\ref{rem:#1}}}

\newcommand{\sectionref}[1]{\hyperref[sec:#1]{Section~\ref{sec:#1}}}

\newcommand{\appendixref}[1]{\hyperref[sec:#1]{Appendix~\ref{sec:#1}}}
\newcommand{\theoremlabel}[1]{\label{thm:#1}}
\newcommand{\theoremref}[1]{\hyperref[thm:#1]{Theorem~\ref{thm:#1}}}

\newcommand{\exampleref}[1]{\hyperref[claim:#1]{Example~\ref{claim:#1}}}

\newcommand{\setupref}[1]{\hyperref[claim:#1]{Setup~\ref{claim:#1}}}

\usepackage[utf8]{inputenc}
\usepackage[T1]{fontenc}
\usepackage{kpfonts}
\usepackage{microtype}

\usepackage{macros}

\usepackage{thmtools,thm-restate}
\newtheorem{theorem}{Theorem}

\newtheorem{lemma}[theorem]{Lemma}
\newtheorem{proposition}[theorem]{Proposition}
\theoremstyle{definition}
\newtheorem{definition}{Definition}

\newtheorem*{property*}{Property}

\usepackage{apptools}
\AtAppendix{\counterwithin{proposition}{section}}
\AtAppendix{\counterwithin{corollary}{section}}
\AtAppendix{\counterwithin{lemma}{section}}

\newenvironment{itm}
{\begin{itemize}[noitemsep,topsep=0pt,parsep=0pt,partopsep=0pt]}
{\end{itemize}}
\newenvironment{enum}
{\begin{enumerate}[noitemsep,topsep=0pt,parsep=0pt,partopsep=0pt]}
{\end{enumerate}}

\title{Algorithmic Decision Making with Conditional Fairness} 
\author[1]{Renzhe Xu}
\author[1]{Peng Cui}
\author[2]{Kun Kuang}
\author[1]{Bo Li}
\author[1]{Linjun Zhou}
\author[1]{Zheyan Shen}
\author[3]{Wei Cui}
\affil[ ]{{\small xrz199721@gmail.com, cuip@tsinghua.edu.cn, kunkuang@zju.edu.cn, libo@sem.tsinghua.edu.cn,}}
\affil[ ]{{\small \{zhoulj16, shenzy17\}@mails.tsinghua.edu.cn, cuiwei@songshuai.com}}
\affil[ ]{}

\affil[1]{Tsinghua University}
\affil[2]{Zhejiang University}
\affil[3]{Squirrel AI Learning}

\date{}

\begin{document}
\maketitle

\begin{abstract}
Nowadays fairness issues have raised great concerns in decision-making systems. Various fairness notions have been proposed to measure the degree to which an algorithm is unfair. In practice, there frequently exist a certain set of variables we term as fair variables, which are pre-decision covariates such as users' choices. The effects of fair variables are irrelevant in assessing the fairness of the decision support algorithm. We thus define conditional fairness as a more sound fairness metric by conditioning on the fairness variables. Given different prior knowledge of fair variables, we demonstrate that traditional fairness notations, such as demographic parity and equalized odds, are special cases of our conditional fairness notations. Moreover, we propose a Derivable Conditional Fairness Regularizer (DCFR)\footnote{The code is available at https://github.com/windxrz/DCFR.}, which can be integrated into any decision-making model, to track the trade-off between precision and fairness of algorithmic decision making. Specifically, an adversarial representation based conditional independence loss is proposed in our DCFR to measure the degree of unfairness. With extensive experiments on three real-world datasets, we demonstrate the advantages of our conditional fairness notation and DCFR.
\end{abstract}

\section{Introduction}
Nowadays fairness issues have raised great concerns in decision-making systems such as loan applications \cite[]{mukerjee2002multi}, hiring processes \cite[]{rivera2012hiring}, and criminal justice \cite[]{larson2016we}. Poorly designed algorithms tend to amplify the bias existed in data, resulting in discriminations towards specific groups of individuals based on their inherent characteristics, which are often named as sensitive attributes in fairness problems. For example, race is a sensitive attribute for crime judgment. ProPublica \cite[]{larson2016we} found it is unfair that African Americans were more likely to be incorrectly labeled as higher risk compared with Caucasians in the COMPAS system. However, what is fair and how to develop fair algorithms for algorithmic decision making are of paramount importance for both academic research and practical applications.

Recently, many works defined their fairness and proposed corresponding fair algorithms, from which, the definition of fairness can be divided into three types: individual fairness \cite[]{dwork2012fairness}, group fairness \cite[]{hardt2016equality,kleinberg2016inherent}, and causality-base fairness notions \cite[]{kusner2017counterfactual, kilbertus2017avoiding, chiappa2019path, nabi2018fair, wu2019pc, russell2017worlds}. Individual fairness requires similar individuals should have similar outcomes. However, it is difficult to define the similarity between individuals. Group fairness requires equity among different groups but they only use sensitive attributes and outcomes as measuring features. As a result, these notions may fail to distinguish between fair and unfair parts in the problem. For example, \citet{pearl2009causality} studied the case of Berkeley's alleged sex bias in graduate admission \cite[]{bickel1975sex} and found that data showed a higher rate of admission for male applicants overall but the result is different when looking into the department choice. The bias caused by department choice should be considered fair but traditional group fairness notions fail to judge fairness since they do not take the department choice into account.

Inspired by this, causality-based fairness notions arise. In these papers, the authors firstly assumed a causal graph between the features, and afterward, they could define the unfair causal effect from the sensitive attribute to the outcome as a metric. However, these fairness notions need very strong assumptions and they are not scalable.

In practice, there frequently exist a certain set of variables we term as fair variables, which are pre-decision covariates such as the department choice in Berkeley's graduate admission problem. The effects of fair variables are irrelevant in assessing the fairness of the decision support algorithm. We thus define conditional fairness as a more sound fairness metric by conditioning on the fairness variables. In detail, outcome variables should be independent of sensitive attributes conditional on these fair variables.

The definition of conditional fairness has several advantages. Firstly, the fair variables can be any variables determined by the decision-makers or decision-making system inspectors, which provides a more flexible judging method. Secondly, conditional fairness can be viewed as a more general fairness notion since it can be easily reduced to demographic parity and equalized odds. If we believe that none of the features are fair variables, the conditional independence constraint becomes a normal independence constraint. If we choose fair variables as the outcome, the constraint is transformed into the target in equalized odds. Thirdly, the definition does not need strong assumptions in causality-based fairness definitions, which makes it much easier to be applied to real problems with a large amount of data and features. 

The main challenge of this definition is to formulate the conditional independence constraint into a derivable loss function, which makes it impossible to be applied to commonly used gradient-based methods. Inspired by the conditional independence testing techniques in causality structure discovery literature \cite[]{zhang2012kernel, strobl2019approximate,daudin1980partial}, we formulate the conditional independence constraint into a computationally amenable form, which serves as a regularizer and is subsequently appended to the prediction loss function. We call the working regularizer as the Derivable Conditional Fairness Regularizer (DCFR).

As fair representation learning is a common in-processing framework to deal with fairness issues, we apply the regularizer and handle the conditional fairness constraint in this framework. With the regularizer, the target optimization problem can be cast as a minimax problem, which can be solved via adversarial learning. We further show that our method is also a general method that can be used in common group fairness notions (demographic parity and equalized odds). Most interestingly, when we believe none of the variables are fair, the target problem is reduced to demographic parity and our method also becomes the same as that proposed by \citet{madras2018learning} when dealing with the same problem.

Finally, we apply our method to real datasets and plot accuracy-fairness curves for three targets (demographic parity, equalized odds and, conditional fairness). We show that our method performs better on the conditional fairness task while having similar results with the state-of-the-art on the other two tasks.

In summary, our contributions are highlighted as follows:
\begin{itm}
    \item By exploring the fair variables, we propose conditional fairness, which is more general than previous definitions on fairness.
    \item We propose a novel Derivable Conditional Fairness Regularizer (DCFR) to optimize conditional fairness by learning conditional independent representation. Our DCFR can be easily integrated into decision-making models to improve fairness.
    \item We demonstrate the effectiveness of our proposed algorithm on fair decision making with three real-world datasets. Furthermore, DCFR performs especially better than baselines when increasing the number of values that fair variables can take.
\end{itm}

\section{Related Works}
There are several common types of fairness notions including individual fairness, group fairness, and causality-based fairness. The most commonly used individual fairness notion is fairness through awareness \cite[]{dwork2012fairness} which requires that similar individuals should be treated similarly. However, it is difficult to define the similarity function between different individuals. Therefore, individual fairness still lacks further research up to today. Group fairness notions require the algorithm should treat different groups of individuals equally. The most commonly used group fairness notions include demographic parity \cite[]{dwork2012fairness}, equal opportunity \cite[]{hardt2016equality}, equalized odds \cite[]{hardt2016equality} and calibration \cite[]{kleinberg2016inherent}. These fairness notions are easy to understand and implement in real machine learning problems. However, they only use sensitive attributes and outcomes as measuring features. As a result, these notions may fail to distinguish between fair and unfair parts in the problem. To define fairness more elaborately, causality-based fairness notions are proposed recently. In these causality-based fairness notions such as counterfactual fairness \cite[]{kusner2017counterfactual}, path-specific counterfactual fairness \cite[]{chiappa2019path, nabi2018fair, wu2019pc}, the authors first define a causal graph among the features, and afterward, they can distinguish the unfair causal effect from the sensitive attribute to the outcome. However, these fair notions need very strong assumptions and they are not scalable.

\citet{kamiran2013quantifying} proposed the most similar fairness notions as us. In \cite[]{kamiran2013quantifying}, the authors defined the variables as explanatory variables and proposed algorithms to mitigate the illegal discrimination they defined. However, their method is limited as it may do great harm to accuracy and it cannot be applied in practice as it cannot provide a tunable tradeoff between fairness and utility. Recently, \citet{dutta2020information} also proposed an information-theoretic decomposition of the total discrimination.

Methods that mitigate biases in the algorithms fall under three categories: pre-processing \cite[]{wang2019repairing, feldman2015certifying, kamiran2012data}, in-processing \cite[]{zafar2017fairness, hashimoto2018fairness}, and post-processing \cite[]{hardt2016equality} algorithms. Representation learning is a common in-processing method which is first proposed by \citet{zemel2013learning}. The authors try to mitigate individual unfairness and demographic discrimination simultaneously. Recently, learning representations via adversary has become the state-of-the-art method. \citet{edwards2015censoring} first proposed this kind of method and they provided a framework to mitigate demographic discrimination. Several works followed this framework such as  \cite[]{beutel2017data, zhang2018mitigating, adel2019one,madras2018learning, zhao2020conditional}. In particular, \citet{madras2018learning} proposed to use different adversarial loss function when faced with different fair notions. \citet{zhao2020conditional} redesigned the loss functions to mitigate the gap of demographic parity and equalized odds simultaneously, which is proved to be difficult in \cite[]{kleinberg2016inherent}. However, these works all focus on the most commonly used group fairness notions. Therefore they cannot be applied to the general conditional fairness target. \citet{agarwal2018reductions} proposed a general method to mitigate any fairness notions that can be written as linear inequalities on conditional moments. But they still require the categorical fair variables which makes it difficult to be extended to more general form.

Conditional independence tests have been popularly used in causal structure discovery problems \cite[]{spirtes2000causation}. In order to deal with more flexible distributions, several novel conditional independence tests have been proposed \cite[]{fukumizu2008kernel, ramsey2014scalable, sejdinovic2013equivalence, strobl2019approximate}. However, these methods cannot be mixed with gradient-based machine learning algorithms, since they usually calculate a statistic first and estimate a p-value with random methods. Our method is based on an equivalent relation of conditional independence \cite[]{daudin1980partial} and is tractable in common machine learning algorithms.
\section{Preliminary}
\subsection{Notations}
We suppose the dataset consists of a tuple $D = (S, X, Y)$, where $S$ represents sensitive attributes such as gender and race, $X$ represents features, and $Y$ represents the outcome. Furthermore, we divide features $X$ into two parts $X = (F, O)$, where $F$ represents fair variables and $O$ represents other features. We use $m_X, m_F, m_O$ to denote the dimension of the features and we have $m_X = m_F + m_O$. We use calligraphic fonts to represent the range of corresponding random variables. For example $\mathcal{X}$ represents the space of $X$ and $\mathcal{X} \subset \mathbb{R}^{m_X}$. Similarly, we have $\mathcal{F} \subset \mathbb{R}^{m_F}$. To simplify, we suppose the sensitive attribute and the outcome are binary, which means $\mathcal{Y}, \mathcal{S} = \{0, 1\}$. We set $S = 1$ as the privileged group and $Y=1$ as the favored outcome.

We suppose there are $N$ samples in total and we use $S_i$, $X_i$, $Y_i$, $F_i$, $O_i$ to represent the features of $i$-th sample. In addition, for a condition $E$, we use $D(E)$ to represent the samples that satisfy the condition and $|D(E)|$ to represent the number of these samples. For example, $D(Y=1)$ means the samples that satisfy $Y_i=1$ and $|D(Y=1)|$ is the total number of such samples.

A fair machine learning problem is to design a fair predictor $\hat{Y}$ with parameters $\boldsymbol{\theta}$ $: \mathcal{X} \times \mathcal{S} \rightarrow \mathcal{Y}$, which maximizes the likelihood $P(Y, X, S | \boldsymbol{\theta})$ while satisfying some specific fair constraints, which we will introduce in the next section.

\subsection{Fairness Notions}
We first introduce some well-known fair notions in machine learning problems.

\begin{definition}[Demographic parity (DP)]
    Given the joint distribution $D$, the classifier $\hat{Y}$ satisfies demographic parity with respect to sensitive attribute $S$ if $\hat{Y}$ is independent of $S$, i.e.\begin{equation}\equationlabel{DP}
        \hat{Y} \perp S.
    \end{equation}
\end{definition}
The definition of DP is clear and concise, representing that S has no predictive power to $\hat{Y}$, but in practice we are also interested in some evaluation metric to reveal how fair the system is. Thus the following equivalent form $\Delta DP$ is proposed to measure the degree of fairness.
\begin{equation}
    \Delta DP \overset{\Delta}{=} |P(\hat{Y} = 1 | S = 1) - P(\hat{Y} = 1 | S = 0)|.
\end{equation}
Easy to show that $\hat{Y} \perp S$ if and only if $\Delta DP = 0$.

One of the drawbacks of $\Delta DP$ is that when the base rate differs significantly among two groups, \textit{i.e.}, $P(Y=1 | S=0) \neq P(Y=1 | S=1)$, the utility could be limited. \citet{hardt2016equality} further proposed another notion Equalized Odds to avoid this problem.

\begin{definition}[Equalized odds (EO)]
    Given the joint distribution $D$, the classifier $\hat{Y}$ satisfies equalized odds with respect to sensitive attribute $S$ if $\hat{Y}$ is independent of $S$ conditional on $Y$, i.e.
    \begin{equation}\equationlabel{EO}
        \hat{Y} \perp S \mid Y.
    \end{equation}
\end{definition}
Similarly, the metric $\Delta EO$ is defined as the expectation of the absolute difference of true positive rate and false positive rate across two groups.
\begin{equation}
    \begin{aligned}
        \Delta EO & \overset{\Delta}{=} \mathbb{E}_y\left[\left|P(\hat{Y} = 1 | S = 1, Y = y) - P(\hat{Y} = 1 | S = 0, Y = y)\right|\right]\\
        & = P(Y=0) \left|P(\hat{Y} = 1 | S = 1, Y = 0) - P(\hat{Y} = 1 | S = 0, Y = 0)\right|\\
        & + P(Y=1) \left|P(\hat{Y} = 1 | S = 1, Y = 1) - P(\hat{Y} = 1 | S = 0, Y = 1)\right|.
    \end{aligned}
\end{equation}

It is also easy to show that $\hat{Y} \perp S \mid Y$ if and only if $\Delta EO = 0$.

None of the fairness notions above take fair variables into account, inspired by \citet{kamiran2013quantifying} and \citet{corbett2017algorithmic}, we denote conditional fairness as 
\begin{definition}[Conditional fairness (CF)]
    Given the joint distribution $D$, the classifier $\hat{Y}$ satisfies conditional fairness with respect to sensitive attribute $S$ and fair variables $F$ if $\hat{Y}$ is independent of $S$ conditional on $F$, i.e.
    \begin{equation}\equationlabel{CF}
        \hat{Y} \perp S \mid F.
    \end{equation}
\end{definition}

In addition, similar to $\Delta EO$, we define a metric $\Delta CF$ as:
\begin{equation}
\equationlabel{eqn:cfdef}
\Delta CF \overset{\Delta}{=} \mathbb{E}_f\left[\left|P(\hat{Y} = 1 | S = 1, F = f) - P(\hat{Y} = 1 | S = 0, F = f)\right|\right]
\end{equation}
Specifically, when fair variables are continuous, \equationref{eqn:cfdef} becomes:
\begin{equation}
    \Delta CF = \int_{f \in \mathcal{F}} \left|P(\hat{Y} = 1 | S = 1, F = f) - P(\hat{Y} = 1 | S = 0, F = f)\right| \mathrm{d} \mathbb{P}(f).
\end{equation}
and when fair variables are categorical, $\Delta CF$ becomes
\begin{equation}\equationlabel{CF_metric}
    \Delta CF = \sum_{f \in \mathcal{F}} \left|P(\hat{Y} = 1 | S = 1, F = f) - P(\hat{Y} = 1 | S = 0, F = f)\right|P(F=f).
\end{equation}
$\Delta CF$ aims to calculate the mean of the absolute difference between two groups among all potential values of the fair variables. Similarly, we have $\hat{Y} \perp S \mid F$ if and only if $\Delta CF = 0$.

\paragraph{Compare CF with DP and EO} On the one hand, conditional fairness can take more complex situations into account. On the other hand, conditional fairness is more general and it can be easily reduced to DP and EO.

\begin{figure}[ht]
    \centering
    \includegraphics[width=0.3\linewidth]{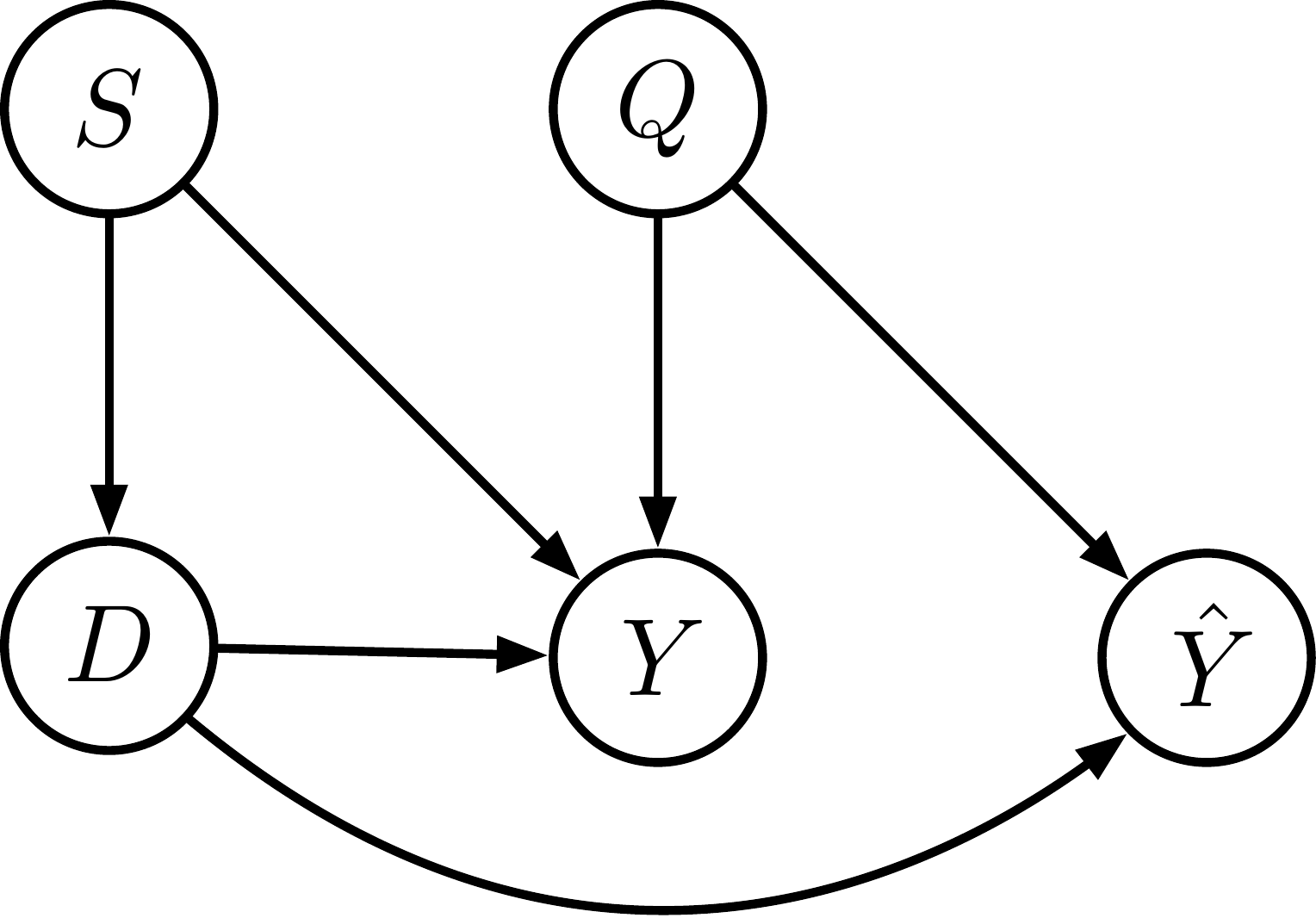}
    \caption{The data-generating graph for a toy example of a college admission case. $S$, $D$, $Q$, $Y$ represent gender, department choice, qualification, and historical admission decision, respectively. $\hat{Y}$ represents a conditional fair decision-making system.}
    \figurelabel{fig:UCB}
\end{figure}

Consider the data-generating graph for a toy example of a college admission case in \figureref{fig:UCB}. Because qualification requirements usually differ among various departments, it is fair to determine outcomes according to department choices and qualifications. Hence any predictors with form $\hat{Y} = f(Q, D)$ can be considered fair in practice. It is easy to show that when setting department choice $D$ as the fair variable, $\hat{Y}$ is conditional fair. However, DP and EO fail to judge the fairness of $\hat{Y}$ as \equationref{DP} and \equationref{EO} may not be satisfied.

In addition, conditional fairness is a more flexible fairness notion as:
\begin{itm}
    \item If we believe none of the features $X$ is fair, which means $F = \emptyset$, the conditional independence target is reduced to the independence condition as shown in \equationref{DP} and conditional fairness is reduced to DP.
    \item If we set $F$ as $Y$, the conditional independence target is reduced to the conditional independence as shown in \equationref{EO} and conditional fairness is reduced to EO.
\end{itm}

\paragraph{Compare CF with causality-based fairness notions} Generally speaking, conditional fairness requires much fewer assumptions than causality-based fairness notions, which makes CF practical in real problems.

Under some circumstances, a conditional fair decision-making system can satisfy causality-based fairness notions. Consider path-specific fairness \cite[]{chiappa2019path} in the example shown in \figureref{fig:UCB}. The directed path $S \rightarrow Y$ can be viewed as an unfair path while $S \rightarrow D \rightarrow Y$ and $Q \rightarrow Y$ are fair paths. Hence, the historical decisions $Y$ is not path-specific fair for the existence of unfair path $S \rightarrow Y$. However, the conditional fair decision-making system $\hat{Y} = f(Q, D)$ successfully satisfies the requirement as the unfair path $S \rightarrow \hat{Y}$ does not exist. As for deeper connections between conditional fairness and causality-based fairness notions, we remain as future works.

\subsection{Problem Formulation}
Next we will apply our definition of conditional fairness into real fair problems. In general, the goal of a fairness problem is to achieve a balance between fairness and algorithm performance. Formally, we need to design a loss function on prediction $L_\text{pred}(\hat{Y}, Y)$ and another loss function on fairness $L_\text{fair}(\hat{Y}, S, F)$. The optimization goal of a fairness problem can be formulated as:
\begin{equation}\equationlabel{eq:problem_formulation}
    \boldsymbol{\theta} = \underset{\boldsymbol{\theta}}{\arg\min}\, L(\hat{Y}) = \underset{\boldsymbol{\theta}}{\arg\min}\, L_\text{pred}(\hat{Y}, Y) + \lambda \cdot L_\text{fair}(\hat{Y}, S, F),
\end{equation}
where the hyper-parameter $\lambda$ provides a trade-off between fairness and performance. When $\lambda$ is large, the target tends to make $L_\text{fair}$ small which can ensure fairness while doing harm to performance, and the result is opposite when $\lambda$ is small.

As for the prediction loss, any form of traditional loss functions are suitable such as cross-entropy or L1 loss. While the fairness loss targeted for conditional fairness is difficult to design relatively. When fair variables are categorical, we can use the $\Delta CF$ metric as a loss function. However, in practice, the fair variables may contain many different values or they may be continuous. Under this circumstance, the metric can no longer be a suitable loss function for optimization. Inspired by this issue, we will propose a new derivable loss function that can deal with these situations in the next section.

\section{Proposed Method}
\begin{figure}[ht]
    \centering
    \includegraphics[width=0.5\linewidth]{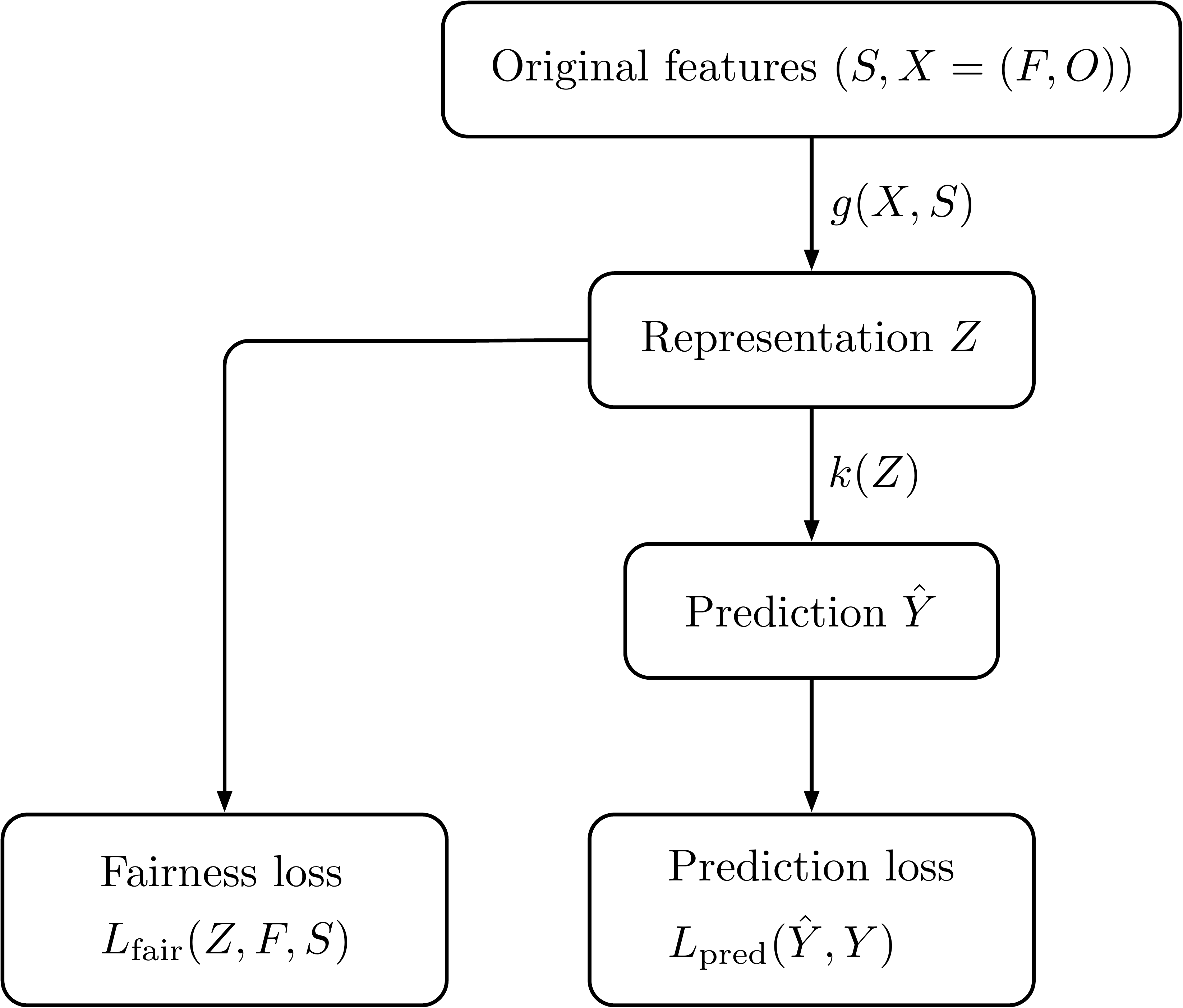}
    \caption{The framework of our method. The variables include sensitive attribute $S$, features $X$, outcome $Y$, representation $Z$, and prediction $\hat{Y}$. $X$ is divided into fair variables $F$ and other variables $O$. The function $g$ maps the original features into the representation space, and function $k$ maps the representations into the outcome space. There are two loss functions that measure utility and fairness respectively.}
    \figurelabel{fig:framework}
\end{figure}

An intuitive way to deal with this conditional independence is to divide the whole training samples into different groups with respect to the value of fair variables and then deal with these groups separately using traditional methods handling naive independence problems. The main drawback is that this method assumes that fair variables are categorical and $|\mathcal{F}|$ is small. Meanwhile, when $|\mathcal{F}|$ becomes very large, representing that fair variables can take many different values, this naive method requires exactly $|\mathcal{F}|$ different models to deal with different subgroups, which has the potential of overfitting due to lack of training data in each subgroup. Furthermore, when fair variables are continuous, it becomes impossible to group by fair variables directly. Hence we need a more general framework to ensure the model's scalability.


Our solution to this problem is to learn a latent representation $Z$, which satisfies condition (\equationref{CF}). Suppose the representation has $m_Z$ dimensions, $g: \mathbb{R}^{m_X} \times \{0, 1\} \rightarrow \mathbb{R}^{m_Z}$ is the function from the space of $X$ and $S$ to representation space. The prediction function $k: \mathbb{R}^{m_Z} \rightarrow [0, 1]$ yields the probability of the sample in the positive class. The framework of our model is shown in \figureref{fig:framework}. We now rewrite \equationref{eq:problem_formulation} under this representation learning framework as:
\begin{equation}\equationlabel{total_loss}
    \boldsymbol{\theta} = \underset{\boldsymbol{\boldsymbol{\theta}}}{\arg\min}\, L_\text{pred}(k(g(X, S)), Y) + \lambda \cdot L_\text{fair}(g(X, S), F, S).
\end{equation}

\subsection{Conditional Independence}

In this section, we first introduce a conditional independence theorem proposed by \citet{daudin1980partial}. Afterward we will transform it into the form that can be applied to fairness problems. Finally we will give a regularizer to measure conditional independence.

\begin{lemma}[Characterization of conditional independence \cite{daudin1980partial}]\lemmalabel{daudin}
    The random variables $Z$, $S$ are independent conditional on $F$ ($Z \perp S \mid F$) if and only if, for any function $u \in L^2_{S}$, $\tilde{h} \in \mathcal{E}_{ZF}$,
    \begin{equation}
        \mathbb{E}[u(S) \cdot \tilde{h}(Z, F)] = 0,
    \end{equation}
    where
    \begin{align}
        L^2_S & = \left\{u(S) \mid \mathbb{E}[u^2] < \infty\right\}\equationlabel{L_S},\\
        L^2_{ZF} & = \left\{h(Z,F) \mid \mathbb{E}[h^2] < \infty\right\}\equationlabel{L_ZF},\\
        \mathcal{E}_{ZF} & = \left\{\tilde{h}(Z,F) \in L^2_{ZF} \mid \mathbb{E}[\tilde{h}|F] = 0\right\}\equationlabel{E_ZF}.
    \end{align}
\end{lemma}

\lemmaref{daudin} is designed for general cases and can be simplified in fairness issues. Considering cases of one single binary sensitive attribute, which means $S$ is binary, the condition can be transformed into the following form.

\begin{proposition}\propositionlabel{prop:CI}
    If random variable $S$ is binary and $S \in \{0, 1\}$, the random variables $Z$, $S$ are independent conditional on $F$ ($Z \perp S \mid F$) if and only if, for any $\tilde{h} \in \mathcal{E}_{ZF}$,
    \begin{equation}\equationlabel{binary}
        \mathbb{E}[\mathbb{I}(S = 1) \cdot \tilde{h}(Z, F)] = 0,
    \end{equation}
    where $\mathcal{E}_{ZF}$ is shown in \equationref{E_ZF} and $\mathbb{I}(S = 1)$ is the indicative function defined as follow:
    \begin{equation}
        \mathbb{I}(S = 1) =
        \begin{cases}
            1, & \text{if}\quad S = 1,\\
            0, & \text{if}\quad S = 0.
        \end{cases}
    \end{equation}
\end{proposition}

However, \propositionref{prop:CI} can hardly be applied to practice directly because of the complexity of function space $\mathcal{E}_{ZF}$ shown in \equationref{E_ZF}. Therefore, we need to transform \equationref{binary} to another form which depends on a simpler function space.

\begin{theorem}[Characterization with binary variable]\theoremlabel{thrm:chct_with_bnry_vrbls}
    If random variable $S$ is binary and $S \in \{0, 1\}$, the random variables $Z$, $S$ are independent conditional on $F$ ($Z \perp S \mid F$) if and only if, for any $h \in L^2_{ZF}$,
    \begin{equation}\equationlabel{Lh}
        Q(h) \overset{\Delta}{=} \mathbb{E}\left[\mathbb{I}(S = 1)P(S=0|F)h(Z,F)\right] - \mathbb{E}\left[\mathbb{I}(S=0)P(S=1|F)h(Z,F)\right] = 0.
    \end{equation}
\end{theorem}

Compared with function space $\mathcal{E}_{ZF}$ in \propositionref{prop:CI}, $L_{ZF}$ space in \theoremref{thrm:chct_with_bnry_vrbls} is much simpler. Based on this theorem, we propose the following regularizer:
\begin{definition}[Derivable Conditional Fairness Regularizer]\definitionlabel{def:regularizer}
    $Z, F, S$ are random variables. $S$ is binary and $S \in \{0, 1\}$. $Q(h)$ is defined in \equationref{Lh}. Define the regularizer $L_\text{fair}(Z,F,S)$
    \begin{equation}\equationlabel{eq:regularizer}
        L_\text{fair}(Z,F,S) \overset{\Delta}{=} \sup_{h \in H_{ZF}} |Q(h)|,
    \end{equation}
    where
    \begin{equation}\equationlabel{eq:H_ZF}
        H_{ZF} = \left\{h \in L^2_{ZF} | 0 \leq h(Z,F) \leq 1\right\}.
    \end{equation}
\end{definition}

The motivations of regularizer are that, firstly, notice that if there exists a function $h \in L^2_{ZF}$ so that $Q(h) \ne 0$, then $\sup_{h \in L^2_{ZF}} |Q(h)|$ can be arbitrarily large. And the $L^2_{ZF}$ space is too large for further analysis. Therefore we first bound the range of $h$ into $[0, 1]$, which produces the $H_{ZF}$ space. Secondly, when $L_\text{fair}(Z,F,S) = 0$, according to \theoremref{thrm:chct_with_bnry_vrbls}, the random variables $Z$ and $S$ are independent conditional on $F$.

Furthermore, we can simplify \equationref{eq:regularizer} by the following theorem.
\begin{theorem}\theoremlabel{thrm:CI_loss}
    $L_\text{fair}(Z,F,S)$, $H_{ZF}$, and $Q(h)$ are defined in \theoremref{thrm:chct_with_bnry_vrbls} and \definitionref{def:regularizer}. Then
    \begin{equation}\equationlabel{CI_loss}
        L_\text{fair}(Z,F,S) = \sup_{h \in H_{ZF}} |Q(h)| = \sup_{h \in H_{ZF}} Q(h).
    \end{equation}
\end{theorem}

\theoremref{thrm:CI_loss} provides a computationally amenable form, which serves as a regularizer and is applied to the prediction loss function.

To better understand $Q(h)$, we transform the \equationref{Lh} as a weighted L1 loss when using $h(Z,F)$ to predict $S$.
\begin{equation}\equationlabel{weighted_expected_L}
    \begin{aligned}
        & Q(h)\\
        = &\mathbb{E}[\mathbb{I}(S = 1)P(S=0|F)h(Z,F)] - \mathbb{E}[\mathbb{I}(S=0)P(S=1|F)h(Z,F)]\\
        = & C - \left[\mathbb{E}[\mathbb{I}(S = 1)P(S=0 | F)(1-h)] + \mathbb{E}[\mathbb{I}(S=0)P(S=1 | F)h]\right],
    \end{aligned}
\end{equation}
where
\begin{equation}\equationlabel{eq:constant}
    C=\mathbb{E}[\mathbb{I}(S = 1)P(S=0|F)]
\end{equation}
is a constant.

The term in the bracket is a weighted L1 loss and conditional probability $P(S=\cdot|F)$ is the weight. Actually, this weight can be learned from finite samples with any non-parametric or parametric algorithms such as regression methods. Therefore our model can be applied to real large datasets with continuous fair variables.

In addition, we give the formulation of $Q(h)$ when fair variables are categorical in order to apply it to simpler circumstances, such as demographic parity or equalized odds. In detail, we can use $|D(S=1-S_i, F=F_i)| / |D(F=F_i)|$ to estimate the weight for the $i$-th sample, therefore \equationref{weighted_expected_L} can be written as
\begin{equation}\equationlabel{weighted_sample_L}
    \begin{aligned}
        Q(h) & \approx C - \frac{1}{N}\left[\sum_{i=1}^N\frac{|D(S=1-S_i, F=F_i)|}{|D(F=F_i)|}|h(F_i, Z_i) - S_i|\right].
    \end{aligned}
\end{equation}

\subsection{Adversarial Learning}
Now we combine the total loss function as shown in \equationref{total_loss} and the conditional independence in \equationref{CI_loss} and get:
\begin{equation}\equationlabel{final_loss}
    \begin{aligned}
        \boldsymbol{\theta} = & \underset{g, k}{\arg\min}\, L_\text{pred}(k(g(X, S)), Y) + \lambda \cdot \sup_h Q(h)\\
        = & \underset{g, k}{\arg\min}\, \sup_h L_\text{pred}(k(g(X, S)), Y) + \lambda \cdot Q(h).
    \end{aligned}
\end{equation}
As $Q(h)$ is actually a weighted L1 loss, the loss function above can be optimized with the method of adversarial learning by setting the $Q(h)$ as the adversarial loss. There are several works that use adversarial learning to solve fairness notions. While the frameworks among these works are similar, the main difference lies in the design of loss functions.

Our method is most closed to LAFTR \cite[]{madras2018learning}. And actually, when $F = \emptyset$, which means the conditional independence constraint $\hat{Y} \perp S \mid F$ is reduced to $\hat{Y} \perp S$, our method is exactly the same as theirs.

Consider $Q(h)$ with finite samples as shown in \equationref{weighted_sample_L}, when $F = \emptyset$, the weight of sample $i$ is actually $|D(S = 1 - S_i)| / N$. Multiply \equationref{weighted_sample_L} with $a = N^2 / \left(|D(S=0)| \cdot |D(S=1)|\right)$ and we get
\begin{equation}\equationlabel{LCFR-DP}
    Q_\text{DP}(h) \overset{\Delta}{=} a \cdot Q(h) \approx C' - \sum_{i=1}^N \frac{1}{|D(S=S_i)|}|h(Z_i) - S_i|,
\end{equation}
which becomes the same as the adversarial loss function provided by \cite{madras2018learning}.

When facing with equalized odds task, we can replace the fair variables $F$ in \equationref{weighted_sample_L} with $Y$ and get:
\begin{equation}\equationlabel{LCFR_EO}
    Q_\text{EO}(h) \overset{\Delta}{=} C - \frac{1}{N}\left[\sum_{i=1}^N\frac{|D(S=1-S_i, Y=Y_i)|}{|D(Y=Y_i)|}|h(Y_i, Z_i) - S_i|\right].
\end{equation}
With \equationref{LCFR-DP} and \equationref{LCFR_EO}, we can apply our method into demographic parity and equalized odds target.

\subsection{Practical Implementation}

\renewcommand{\algorithmicrequire}{\textbf{Input:}}
\renewcommand{\algorithmicensure}{\textbf{Output:}}
\begin{algorithm}
    \caption{\emph{Derivable Conditional Fairness Regularizer (DCFR)}}
    \algorithmlabel{alg}
    \begin{algorithmic}[1]
        \REQUIRE Dataset $D=(X,Y,S)$, $X = (F,O)$, EPOCH, BATCH\_SIZE, ADV\_STEPS.
        \ENSURE $g$, $k$, $h$ as in \equationref{final_loss}
        \STATE /* Step I */
        \FOR{epoch\_i $\leftarrow$ 1 to EPOCH}
            \STATE Random mini-batch $D' = (X'=(F', O'), Y', S')$ from $D$.
            \STATE Freeze $h$. Unfreeze $g, k$.
            \STATE Optimize $g, k$ with gradient descent according to $D'$.
            \STATE Freeze $g, k$. Unfreeze $h$.
            \FOR{adv\_step $\leftarrow$ 1 to ADV\_STEPS}
                \STATE Optimize $h$ with gradient descent according to $D'$.
            \ENDFOR
        \ENDFOR
        \STATE
        \STATE /* Step II */
        \STATE Freeze $g, h$. Unfreeze $k$.
        \FOR{epoch\_i $\leftarrow$ 1 to EPOCH}
            \STATE Random mini-batch $D' = (X'=(F', O'), Y', S')$ from $D$.
            \STATE Optimize $k$ with gradient descent according to $D'$.
            \IF{accuracy on validation set does not increase for continuous 20 epochs}
                \STATE Break.
            \ENDIF
        \ENDFOR
        \RETURN $g, k, h$.
    \end{algorithmic}
\end{algorithm}

In practice, we cannot enumerate all the functions in $H_{ZF}$, we use an MLP with sigmoid as an estimation. Furthermore, we find it difficult to optimize with L1 loss as we use sigmoid functions to bound the $h$ into $[0, 1]$ and this can result in vanishing gradient problem in practice. Instead, We define the L2 loss function $Q'(h)$ as the surrogate of $Q(h)$ and the corresponding conditional independence regularizer $L'_\text{fair}$.
\begin{equation}
    Q'(h) \overset{\Delta}{=} C - \left[\mathbb{E}\left[\mathbb{I}(S = 1)P(S=0|F) \cdot (1-h)^2\right] + \mathbb{E}\left[\mathbb{I}(S=0)P(S=1|F) \cdot h^2\right]\right],
\end{equation}
where $C$ is the constant defined in \equationref{eq:constant}. The corresponding regularizer is
\begin{equation}
    L'_\text{fair} \overset{\Delta}{=} \sup_{h \in H_{ZF}} Q'(h).
\end{equation}

To show that directly optimizing L2 loss could also reach our target, we give the following theorem.
\begin{theorem}\theoremlabel{thrm:L2}
    $L'_\text{fair}$ provides an upper bound of $L_\text{fair}$, i.e.
    \begin{equation}
        L'_\text{fair} \ge L_\text{fair}.
    \end{equation}
\end{theorem}

With this theorem, it makes sense to directly optimize with L2 loss, as L1 loss will decrease synchronously with L2 loss during the optimization process. Using L2 loss instead of L1 loss makes the algorithm much easier to converge in real experiments.

Our algorithm has two steps. We train the model $g$, $h$, $k$ adversarially firstly, and afterward we fine-tune the function $k$ for better performance. During the training step, for each sampled mini-batch, we train predictor part $g$ and $k$ once and train adversarial part $h$ for several times. The number of adversarial steps is also a hyper-parameter. During the fine-tuning step, we run the models with early stop when the accuracy on the validation set does not increase for continuous 20 epochs. The pseudo-code is shown in \algorithmref{alg}.

\section{Experiments}
In this section, we provide the experimental settings and verify the effectiveness of our method in multiple real datasets.
\begin{figure*}[t]
    \centering
    \includegraphics[width=0.9\textwidth]{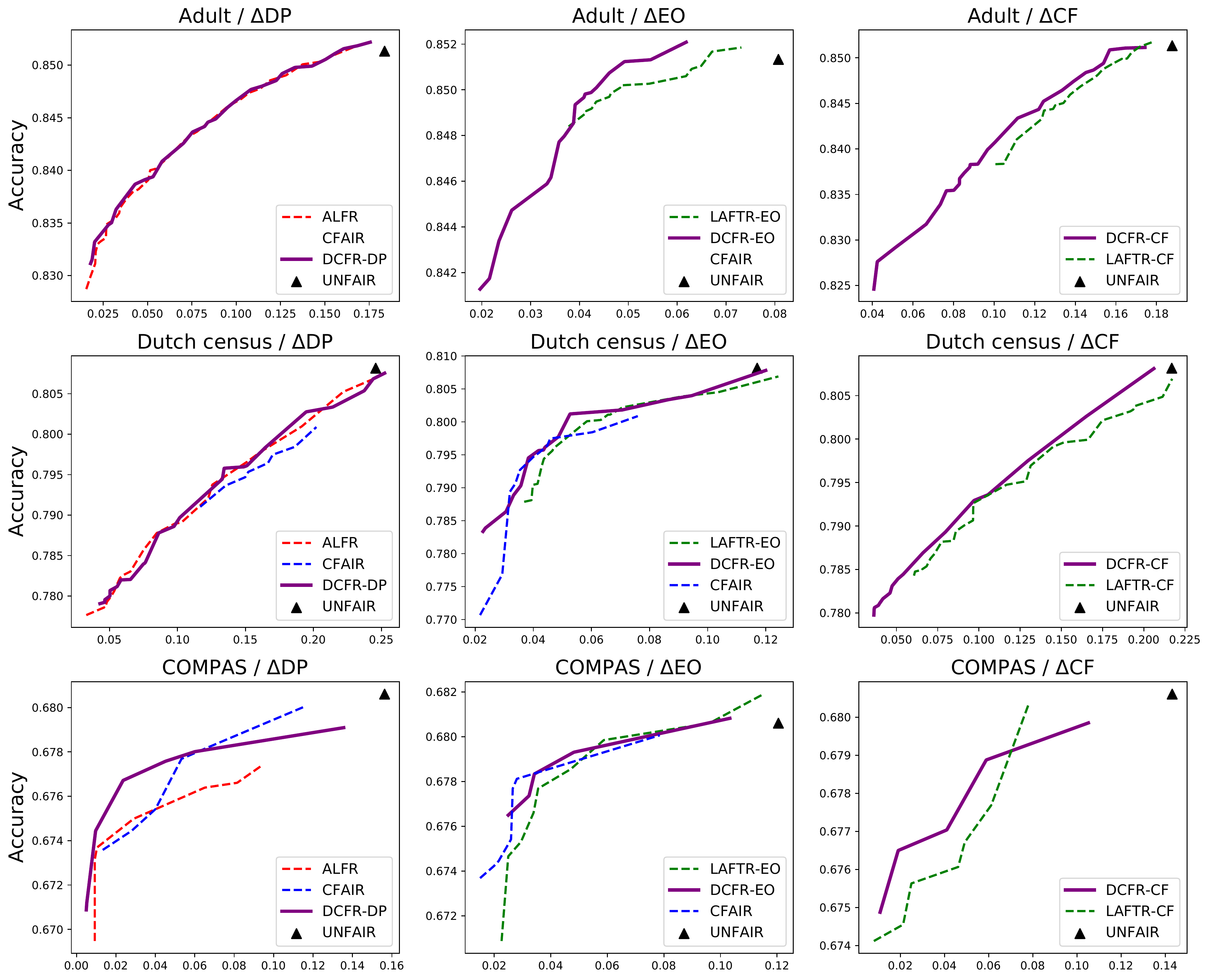}
    \caption{The accuracy-fairness trade-off curves for different fairness metrics ($\Delta DP$, $\Delta EO$, $\Delta CF$ from left to right) on various datasets (Adult, Dutch census, COMPAS dataset from top to bottom, with $|\mathcal{F}|=14,7,2$ respectively). The upper-left corner is preferred. Our method is shown in bolded lines. The UNFAIR algorithm is a triangle mark while other baselines are in dashed lines. We take different values of $\lambda$ from 0.1 to 20, get the mean of accuracy and fairness metric across 5 runs for each model, and plot the Pareto front on the test dataset. While our model performs similarly on $\Delta DP$ and $\Delta EO$ task with baselines, with the increase of $|\mathcal{F}|$, our method performs much better than baselines on $\Delta CF$ task. Note that we do not plot the curve of CFAIR in the Adult dataset because the curve goes beyond the axis range.}
    \figurelabel{fig:result}
\end{figure*}

\subsection{Datasets}
We perform experiments on three real-world datasets that are widely used in fair machine learning problems, including the Adult dataset \cite[]{Dua2019uci}, the Dutch census dataset \cite[]{impus2019}, and the COMPAS dataset \cite[]{larson2016we}.

\begin{itm}
    \item \textbf{Adult}: The goal of the Adult dataset is to predict whether a person makes more than \$50k per year or not. Each instance contains 112 attributes including sex, gender, education level, occupation, etc. In our experiments, we set gender as the sensitive attribute, and consider occupation (with 14 possible categorical values) as the fair variable. The target variable (income) is binary and we set "$\ge$ \$50k per year" as the favored outcome.
    \item \textbf{Dutch census}: This dataset is sampled from the Dutch census dataset, which is conducted by Statistics Netherlands to predict whether a person has a prestigious occupation. Each instance contains 35 attributes including age, gender, marital status, etc. In our experiments, we set gender as the sensitive attribute and level of educational attainment (with 7 possible categorical values) as the fair variable. The target variable is binary and we set "having a prestigious occupation" as the favored outcome.
    \item \textbf{COMPAS}: The COMPAS dataset aims to predict whether a criminal defendant will recidivate within two years or not. Each instance contains 11 attributes including age, race, gender, number of prior crimes, etc. In our experiments, we set race as the sensitive attribute and set the charge degree (with 2 possible categorical values) as the fair variable. The target variable (recidivism or not) is binary and we define "not recidivism" as the favored outcome.
\end{itm}

As a summary, the basic statistics of the datasets are listed in \tableref{tab:dataset}.
\begin{table}[ht]
    \centering
    \caption{Basic statistics of the datasets}
    \tablelabel{tab:dataset}
    \begin{tabular}{ccccccc}
        \toprule
        Dataset & Train/Test & $P(S=0)$ & $P(Y=1)$ & $|\mathcal{F}|$\\
        \midrule
        \textbf{Adult} & 30,162/15,060 & 0.325 & 0.248 & 14\\
        \textbf{Dutch census} & 27,060/11,595 & 0.492 & 0.521 & 7\\
        \textbf{COMPAS} & 4,321/1,851 & 0.659 & 0.455 & 2\\
        \bottomrule
    \end{tabular}
\end{table}

\subsection{Baselines}

As adversarial representation learning has become a prominent solver for fairness-related constrained optimization problems, we also adopt it to solve our target problem. For a fair comparison, we mainly select the following state-of-the-art fairness optimization algorithms that are also solved by adversarial representation learning as baseline methods.

\begin{itm}
    \item \textbf{UNFAIR}: We design a baseline predictive model without any fairness constraint by setting $\lambda$ to be 0 in \equationref{final_loss}.
    \item \textbf{ALFR} \cite[]{edwards2015censoring}: ALFR is specifically designed for demographic parity problems.
    \item \textbf{CFAIR} \cite[]{zhao2020conditional}: CFAIR aims to mitigate the gap of demographic parity and equalized odds simultaneously.
    \item \textbf{LAFTR} \cite[]{madras2018learning}: LAFTR consists of two different loss functions, which target demographic parity and equalized odds respectively. Therefore, we implement two variants \textbf{LAFTR-DP} and \textbf{LAFTR-EO}.
\end{itm}

For general conditional fairness, since none of the methods above propose the method to handle this situation explicitly, we extend the method of LAFTR-EO by replacing the conditional target $Y$ as $F$ in adversarial loss, namely $\textbf{LAFTR-CF}$. In detail, the original target adversarial loss function in LAFTR is
\begin{equation}
    L_{A d v}^{E O}(h)=2-\sum_{i=1}^N \frac{1}{|D(Y=Y_i, S=S_i)|}|h(Z_i)-S_i|.
\end{equation}
We transform it into conditional fairness setting as
\begin{equation}\equationlabel{laftr_cf}
    L_{A d v}^{CF}(h)=|\mathcal{F}|-\sum_{i=1}^N \frac{1}{|D(F=F_i, S=S_i)|}|h(Z_i)-S_i|.
\end{equation}
The differences between the equation above and our methods lie on the input of function $h$ and the sample weight. \equationref{laftr_cf} tends to assign relatively high weights to minority groups with the same $F$ compared with majority groups, which may lose stability while our method treats different groups divided by $F$ equally. 

As conditional fairness is a general notion that encompasses the demographic parity and equalized odds, we implement the following three variants of our method:

\begin{itm}
    \item \textbf{DCFR-DP}: We transform our method by setting $\mathcal{F}$ to be a null set and optimize it by \equationref{LCFR-DP}, so that it can be used directly to solve demographic parity problems.
    \item \textbf{DCFR-EO}: We transform our method by setting $\mathcal{F}$ to be $\mathcal{Y} = \{0,1\}$ and optimize it by \equationref{LCFR_EO}, so that it can be used directly to solve equalized odds problems.
    \item \textbf{DCFR-CF}: We use the general form of our method to solve general conditional fairness problems.
\end{itm}

For demographic parity, we compare our DCFR-DP with ALFR, CFAIR, and UNFAIR. For equalized odds, we compare our DCFR-EO with LAFTR-EO, CFAIR, and UNFAIR, while ALFR cannot handle this fairness target. For conditional fairness, we mainly compare our DCFR-CF with LAFTR-CF. Note that CFAIR method can hardly be applied to conditional fairness target as it requires $|\mathcal{F}|$ different adversarial predictors when calculating adversarial loss, which is impractical in real problems. For the sake of fair comparison and easier convergence, we replace L1 loss with L2 loss function for adversary losses in LAFTR model and our model DCFR.
In addition, we use cross-entropy loss as the prediction loss function in our model DCFR.

As fair variables are categorical in these experiments, we use $\Delta DP$, $\Delta EO$, $\Delta CF$ as evaluation metrics, and smaller values of these metrics mean higher fairness. More experimental details are shown in the appendix.

\subsection{Results}
The results are shown in \figureref{fig:result}. The columns show the accuracy-fairness trade-off curves for demographic parity, equalized odds, and conditional fairness respectively, and the rows correspond to different datasets.

For the tradeoff curves, there are two observation points.
\begin{enum}
    \item If a curve is closer to the left-top point than other curves in the majority range of an evaluation metric, the corresponding method is better. Because it means that given a certain degree of fairness, the method can achieve higher prediction accuracy, while given a certain prediction accuracy, the method can achieve better fairness.
    \item As a fair algorithm, it is important to evaluate how much fairness can be achieved, which is indicated by the left-end point of a curve.
\end{enum}

From \figureref{fig:result}, we can get the following observations.

\begin{itm}
    \item For the conditional fairness task, it is obvious that our DCFR-CF is more advantageous than LAFTR-CF in the sense of both two observation points. In the COMPAS dataset, both methods can reach similar fairness ranges, while our DCFR-CF can achieve better trade-off performance. In Adult and Dutch census datasets, the two curves are close, while DCFR-CF can reach a higher fairness region than LAFTR-CF, which is more obvious in the Adult dataset. The plausible reason is that a larger $|\mathcal{F}|$ in Adult will make the limitation of LAFTR-CF more obvious as it is designed in the context of one single binary conditional variable. 
    \item For the demographic parity and equalized odds tasks, the degenerated variants of our method produce comparable performances with state-of-the-art baselines that are specifically designed for these tasks. In some datasets, our method reports even better results, for example, in the Adult dataset of $\Delta EO$ setting and COMPAS dataset in $\Delta DP$ setting. We attribute this to the strong expressive ability of DCFR.
    \item The overall performance of CFAIR is not satisfactory, especially in the Adult dataset where the curve of CFAIR goes beyond the axis range. We notice that, as shown in \tableref{tab:dataset}, $Y$ is seriously biased in the Adult dataset. In CFAIR, however, the balanced error rate is used in optimization. 
\end{itm}
\section{Conclusions}
In this work, we propose the conditional fairness concerning fair variables and show that it is a general fairness notion with several practical reductions. However, conditional fairness is difficult to optimize directly as it cannot be written as a derivable loss function straightforwardly especially when fair variables are continuous or contain many categorical values. Inspired by conditional independence test methods, we derive an equivalent condition of conditional independence under fairness settings. Based on the equivalent condition, we propose a conditional independence regularizer that can be integrated into gradient-based methods, namely Derivable Conditional Fairness Regularizer (DCFR). We apply the regularizer into the representation learning framework and solve it with adversarial learning. We validate the effectiveness of our method on real datasets and achieve good performance on conditional fairness targets. It is worth mentioning that our method becomes much better than baselines when the number of potential values of fair variable increases.

Potential future work is to apply our method into unsupervised settings as the conditional fairness notion does not rely on $Y$ under most circumstances. With the information of $F$, we can ideally get a more elaborate representation compared with demographic parity. Besides, we find it difficult to measure the performance between different models. On the one hand, the target fairness notions of various models are usually different, which makes it impossible to compare with each other. On the other hand, even for the same fairness target, the most common practice is to plot the fairness-utility trade-off curve, which cannot become an accurate metric. This issue remains open and we believe it is worthwhile to study on.

\section*{Acknowledgment}
This work was supported in part by National Key R\&D Program of China (No. 2018AAA0102004, No. 2018AAA0101900), National Natural Science Foundation of China (No. U1936219, 61772304, U1611461, 71490723, 71432004), Beijing Academy of Artificial Intelligence (BAAI), the Fundamental Research Funds for the Central Universities, Tsinghua University Initiative Scientific Research Grant (No. 2019THZWJC11), Science Foundation of Ministry of Education of China (No. 16JJD630006), and a grant from the Institute for Guo Qiang, Tsinghua University. All opinions, findings, conclusions and recommendations in this paper are those of the authors and do not necessarily reflect the views of the funding agencies.

\clearpage

\bibliographystyle{plainnat}
\bibliography{references}

\clearpage
\appendix
\newtheorem*{theorem*}{Theorem}
\newtheorem*{proposition*}{Proposition}

\section{Proofs}
\subsection{Proof of \texorpdfstring{\propositionref{prop:CI}}{Proposition}}
\begin{proposition*}
    If random variable $S$ is binary and $S \in \{0, 1\}$, the random variables $Z$, $S$ are independent conditional on $F$ ($Z \perp S \mid F$) if and only if, for any $\tilde{h} \in \mathcal{E}_{ZF}$,
    $$
        \mathbb{E}[\mathbb{I}(S = 1) \cdot \tilde{h}(Z, F)] = 0,
    $$
    where $\mathcal{E}_{ZF}$ is shown in \equationref{E_ZF} and $\mathbb{I}(S = 1)$ is the indicative function defined as follow:
    $$
        \mathbb{I}(S = 1) =
        \begin{cases}
            1, & \text{if}\quad S = 1,\\
            0, & \text{if}\quad S = 0.
        \end{cases}
    $$
\end{proposition*}
\begin{proof}
    On the one hand, $\mathbb{I}(S=1) \in L_S^2$. Thus when $Z$, $S$ are independent conditional on $F$, for any $\tilde{h} \in \mathcal{E}_{ZF}$, according to \lemmaref{daudin}, $\mathbb{E}[\mathbb{I}(S=1) \cdot \tilde{h}]=0$.

    On the other hand, $S \in \{0, 1\}$. Therefore for any $u \in L_S^2$, function $u$ can be expressed as:
    $$
        \begin{aligned}
            u(S) & = a \cdot \mathbb{I}(S = 0) + b \cdot \mathbb{I}(S = 1)\\
            & = a + (b - a) \cdot \mathbb{I}(S = 1),
        \end{aligned}
    $$
    where $a, b \in \mathbb{R}$. Thus for any $\tilde{h} \in \mathcal{E}_{ZF}$, when \equationref{binary} is satisfied,
    $$
        \begin{aligned}
            \mathbb{E}[u \cdot \tilde{h}] & = \mathbb{E}\left[\left(a + (b - a)\cdot \mathbb{I}(S=1)\right)\tilde{h}\right]\\
            & = a \cdot \mathbb{E}[\tilde{h}] + (b - a) \cdot \mathbb{E}[\mathbb{I}(S=1)\cdot \tilde{h}]\\
            & = a \cdot \mathbb{E}[\mathbb{E}[\tilde{h} | F]]\\
            & = 0.
        \end{aligned}
    $$
    As a result, according to \lemmaref{daudin}, $Z$ and $S$ are independent conditional on $F$.
\end{proof}

\subsection{Proof of \texorpdfstring{\theoremref{thrm:chct_with_bnry_vrbls}}{Theorem}}
\begin{theorem*}
    If random variable $S$ is binary and $S \in \{0, 1\}$, the random variables $Z$, $S$ are independent conditional on $F$ ($Z \perp S \mid F$) if and only if, for any $h \in L^2_{ZF}$,
    $$
        \begin{aligned}
            Q(h) & \overset{\Delta}{=} \mathbb{E}\left[\mathbb{I}(S = 1)P(S=0|F)h(Z,F)\right]\\
            & - \mathbb{E}\left[\mathbb{I}(S=0)P(S=1|F)h(Z,F)\right] = 0.
        \end{aligned}
    $$
\end{theorem*}
\begin{proof}

    On the one hand, if $Z \perp S \mid F$, according to \propositionref{prop:CI}, for any $\tilde{h} \in \mathcal{E}_{ZF}$, $\mathbb{E}[\mathbb{I}(S=1) \cdot \tilde{h}]=0$. Then for any function $h \in L_{ZF}$, define the corresponding $\tilde{h}$ as
    $$
        \tilde{h}(Z, F) = h(Z, F) - \mathbb{E}[h | F].
    $$
    Because
    $$
        \mathbb{E}[\tilde{h} | F] = \mathbb{E}\left[[h(Z, F) - \mathbb{E}[h | F]] | F\right] = \mathbb{E}[h | F] - \mathbb{E}[h | F] = 0,
    $$
    we can know that $\tilde{h} \in \mathcal{E}_{ZF}$. Hence, we have
    $$
        \begin{aligned}
            Q(h) = & \mathbb{E}[\mathbb{I}(S = 1)P(S=0|F) \cdot h] - \mathbb{E}[\mathbb{I}(S=0)P(S=1|F) \cdot h]\\
            = & \mathbb{E}[\mathbb{I}(S = 1) \cdot h] - \mathbb{E}[P(S = 1 | F) \cdot h]\\
            = & \mathbb{E}[\mathbb{I}(S = 1) \cdot h] - \mathbb{E}\left[\mathbb{E}[P(S = 1 | F) \cdot h | F] \right]\\
            = & \mathbb{E}[\mathbb{I}(S = 1) \cdot h] - \mathbb{E}\left[ P(S = 1 | F) \cdot \mathbb{E}[h | F]\right]\\
            = & \mathbb{E}[\mathbb{I}(S = 1) \cdot h] - \mathbb{E}\left[\mathbb{E}[\mathbb{I}(S = 1) | F] \cdot \mathbb{E}[h | F]\right]\\
            = & \mathbb{E}[\mathbb{I}(S = 1) \cdot h] - \mathbb{E}\left[\mathbb{E}\left[\mathbb{I}(S = 1)\cdot \mathbb{E}[h | F]| F\right]\right]\\
            = & \mathbb{E}\left[\mathbb{I}(S = 1) \cdot h - \mathbb{I}(S = 1)\cdot \mathbb{E}[h | F]\right]\\
            = & \mathbb{E}[\mathbb{I}(S = 1) \cdot \tilde{h}] = 0.\\
        \end{aligned}
    $$
    
    On the other hand, if for any $h \in L_{ZF}$, $Q(h) = 0$, then consider any function $\tilde{h} \in \mathcal{E}_{ZF}$. Similarly, we can get
    $$
        \mathbb{E}[\mathbb{I}(S = 1) \cdot \tilde{h}] = Q(\tilde{h}) = 0.
    $$
    According to \propositionref{prop:CI}, $Z$ and $S$ are independent conditional on $F$.
\end{proof}

\subsection{Proof of \texorpdfstring{\theoremref{thrm:CI_loss}}{Theorem}}
\begin{theorem*}
    $L_\text{fair}(Z,F,S)$, $H_{ZF}$, and $Q(h)$ are defined in \theoremref{thrm:chct_with_bnry_vrbls} and \definitionref{def:regularizer}. Then
    $$
        L_\text{fair}(Z,F,S) = \sup_{h \in H_{ZF}} |Q(h)| = \sup_{h \in H_{ZF}} Q(h).
    $$
\end{theorem*}
\begin{proof}
    Because
    $$
        \begin{aligned}
            & Q(h) + Q(1 - h) \\
            =  & \mathbb{E}[\mathbb{I}(S = 1)P(S=0|F)] - \mathbb{E}[\mathbb{I}(S=0)P(S=1|F)]\\
            = & \mathbb{E}\left[\mathbb{E}[\mathbb{I}(S = 1)P(S=0|F) | F]\right] - \mathbb{E}\left[\mathbb{E}[\mathbb{I}(S = 0)P(S=1|F) | F]\right]\\
            = & \mathbb{E}\left[P(S=0|F)\mathbb{E}[\mathbb{I}(S = 1) | F]\right] - \mathbb{E}\left[P(S=1|F)\mathbb{E}[\mathbb{I}(S = 0) | F]\right]\\
            = & \mathbb{E}\left[P(S=0|F)P(S=1|F)] - \mathbb{E}\right[P(S=0|F)P(S=1|F)] = 0,
        \end{aligned}
    $$
    we have
    $$
        \sup_{h \in H_{ZF}} Q(h) = \inf_{h \in H_{ZF}} Q(h)
    $$
    Hence, 
    $$
        \sup_{h \in H_{ZF}} |Q(h)| = \max\left\{\sup_{h \in H_{ZF}} Q(h), \inf_{h \in H_{ZF}} Q(h)\right\} = \sup_{h \in H_{ZF}} Q(h).
    $$
\end{proof}

\subsection{Proof of \texorpdfstring{\theoremref{thrm:L2}}{Theorem}}
\begin{theorem*}
    $L'_\text{fair}$ provides an upper bound of $L_\text{fair}$, i.e.
    $$
        L'_\text{fair} \ge L_\text{fair}.
    $$
\end{theorem*}
\begin{proof}
    Since $h \in H_{ZF}$, we have $0 \leq h \leq 1$. Therefore $h \ge h^2$ and $1 - h \ge (1 - h)^2$. As a result, for any $h \in H_{ZF}$,
    $$
        \begin{aligned}
            & Q'(h)\\
            = & C - \left[\mathbb{E}[\mathbb{I}(S = 1)P(S=0|F) \cdot (1-h)^2] + \mathbb{E}[\mathbb{I}(S=0)P(S=1|F) \cdot h^2]\right]\\
            \geq & C - \left[\mathbb{E}[\mathbb{I}(S = 1)P(S=0|F) \cdot (1-h)] + \mathbb{E}[\mathbb{I}(S=0)P(S=1|F) \cdot h]\right]\\
            = & Q(h).
        \end{aligned}
    $$
    Finally we get
    $$
        L'_\text{fair} = \sup_{h \in H_{ZF}} Q'(h) \geq \sup_{h \in H_{ZF}} Q(h) = L_\text{fair}.
    $$
\end{proof}

\section{Experimental Details}
\begin{table}[ht]
    \centering
    \caption{Hyper-parameters in the experiments}
    \begin{tabular}{llll}
        \toprule
        & Adult & COMPAS & Dutch \\
        \midrule
        \# of hidden units in prediction & 60 & 8 & 35 \\
        \# of hidden units in adversary & 50 & 8 & 20 \\
        \# of adversarial steps & 10 & 5 & 5 \\
        Batch size & 512 & 256 & 512 \\
        Epoch & 400 & 400 & 400 \\
        Learning algorithm & Adadelta & Adadelta & Adadelta \\
        Learning rate & 1.0 & 1.0 & 1.0 \\
        \bottomrule
    \end{tabular}
    \tablelabel{tab:hyper}
\end{table}

We fix the baseline network architectures so that they are shared among different methods. In detail, we set UNFAIR as a single hidden layer MLP with ReLU as activation function and logistic regression as the outcome function. For the adversary part of CFAIR, LAFTR, and ALFR, we also use a single hidden layer MLP. Its input is the hidden layer in UNFAIR and we apply logistic regression to the outcome. As for our method, we add another $|\mathcal{F}|$ units to the input of the adversarial network. The information about the hyper-parameters is shown in \tableref{tab:hyper}.

To get the accuracy-fairness trade-off curve, we sweep across the coefficient $\lambda$ in \equationref{final_loss} from 0.1 to 20. For each coefficient and each model, we train and fine-tune it for 5 times and get the mean of accuracy and fairness metric on the test set. Finally, we calculate the Pareto front of these results as commonly used in literatures \cite[]{madras2018learning, agarwal2018reductions}.

\end{document}